\documentclass[10pt]{article} 
\usepackage[preprint]{tmlr}

\usepackage{amsmath,amsfonts,bm}

\def\eqref#1{equation~\ref{#1}}

\def\1{\bm{1}}

\DeclareMathAlphabet{\mathsfit}{\encodingdefault}{\sfdefault}{m}{sl}
\SetMathAlphabet{\mathsfit}{bold}{\encodingdefault}{\sfdefault}{bx}{n}

\usepackage{graphicx}
\usepackage{hyperref}
\usepackage{url}
\usepackage{tabularx}
\usepackage{booktabs}
\usepackage{xcolor}
\usepackage{threeparttable}
\usepackage{booktabs}
\usepackage{amsthm}
\theoremstyle{plain}
\newtheorem{lemma}{Lemma}[section]

\title{Rotary Positional Embeddings as Phase Modulation: Theoretical Bounds on the RoPE Base for Long-Context Transformers\\}

\author{\name Feilong Liu\\ \email drbruceliu@gmail.com\\
      \addr https://www.linkedin.com/in/feilong-liu-19b6b18/}

\def\month{MM}  
\def\year{YYYY} 
\def\openreview{\url{https://openreview.net/forum?id=XXXX}} 

\begin{document}

\maketitle

\begin{abstract}
Rotary positional embeddings (RoPE) are widely used in large language models to encode token positions through multiplicative rotations, yet their behavior at long context lengths remains poorly characterized. In this work, we reinterpret RoPE as phase modulation applied to a bank of complex oscillators, enabling analysis through classical signal processing theory.

Under this formulation, we derive principled lower bounds on the RoPE base parameter that are necessary to preserve positional coherence over a target context length. These include a fundamental aliasing bound, analogous to a Nyquist limit, and a DC-component stability bound that constrains phase drift in low-frequency positional modes. We further extend this analysis to deep transformers, showing that repeated rotary modulation across layers compounds angular misalignment, tightening the base requirement as depth increases.

Complementing these results, we derive a precision-dependent upper bound on the RoPE base arising from finite floating-point resolution. Beyond this limit, incremental phase updates become numerically indistinguishable, leading to positional erasure even in the absence of aliasing. Together, the lower and upper bounds define a precision- and depth-dependent feasibility region—a “Goldilocks zone”—for long-context transformers.

We validate the framework through a comprehensive case study of state-of-the-art models, including LLaMA, Mistral, and DeepSeek variants, showing that observed successes, failures, and community retrofits align closely with the predicted bounds. Notably, models that violate the stability bound exhibit attention collapse and long-range degradation, while attempts to scale beyond one million tokens encounter a hard precision wall independent of architecture or training.

Our analysis establishes RoPE base selection as a fundamental necessary architectural constraint, rather than a tunable hyperparameter, and provides practical guidance for designing, scaling, and retrofitting long-context transformers under realistic numerical limits.
\end{abstract}

\section{ Introduction}

The ability to model long‑range dependencies is central to the recent progress of large language models. As context lengths extend from thousands to hundreds of thousands of tokens, the choice and parametrization of positional encodings have become a primary bottleneck in both training stability and inference reliability. Several approaches have been proposed to improve length generalization, including sinusoidal embeddings, learned positional embeddings, attention with linear biases (ALiBi), and rotary positional encoding (RoPE) \citep{vaswani2017attention, su2021roformer, press2022alibi}. Among these, rotary positional embeddings have emerged as a de facto standard in modern transformer‑based language models \citep{touvron2023llama, deepseekv2} due to their parameter efficiency, extrapolation properties, and compatibility with attention mechanisms. Yet despite their widespread adoption, the fundamental limits governing RoPE at long context lengths remain poorly understood.

Empirically, RoPE‑based models exhibit a range of failure modes as context length increases, including degradation in positional discrimination, sensitivity to base scaling, and brittle behavior across model depth and numerical precision. To mitigate these issues, several studies have proposed heuristic modifications to RoPE—such as rescaling the base, remapping frequency spectra, or interpolating positional phases—to extend usable context lengths \citep{su2021roformer, chen2023pi, peng2023yarn}. While often effective, these approaches are primarily empirical and lack a unifying theoretical framework explaining why certain adjustments succeed or fail.

Long-context stress tests, including Needle-in-a-Haystack \citep{kamradt2023needle, zhong2025ropeattention} and RULER \citep{hsieh2024ruler}, have systematically investigated these failure modes. They reveal attention collapse and retrieval errors at extreme sequence lengths, providing empirical benchmarks that complement our theoretical stability and precision bounds. Similar patterns are also predicted by scaling law analyses, which quantify performance degradation when models are evaluated beyond their training distribution \citep{hernandez2021scaling}.

A key challenge in analyzing RoPE is that positional information is injected multiplicatively through rotations rather than additively through embedding vectors. This distinction complicates direct comparison with classical sinusoidal or learned positional encodings and obscures the accumulation of positional distortion across transformer layers. As a result, existing analyses typically reason about RoPE geometrically or layer-locally, without explicitly modeling how phase errors propagate with depth or interact with finite numerical precision.

In this work, we propose a complementary perspective: we reinterpret RoPE as phase modulation applied to a bank of complex oscillators with geometrically spaced frequencies. This signal-processing formulation makes the spectral structure of RoPE explicit and enables analysis using well-established tools from harmonic analysis and numerical computation. Under this view, positional encoding fidelity is governed by the stability of phase relationships across frequencies, layers, and finite-precision arithmetic.

This perspective yields several concrete insights. First, we derive necessary lower bounds on the RoPE base parameter required to preserve positional coherence over a target context length. These include a fundamental aliasing constraint, analogous to a Nyquist limit, and a stability constraint on the lowest-frequency (DC-aligned) modes that dominate long-range alignment. Second, we show that in deep transformers, repeated rotary modulation compounds angular misalignment across layers, tightening the base requirement as depth increases—even when per-layer distortions are small. Third, we derive a precision-dependent upper bound on the RoPE base induced by finite floating-point resolution, beyond which incremental phase updates become numerically indistinguishable, leading to positional erasure.

Taken together, these lower and upper bounds define a feasibility region for RoPE base selection that depends jointly on context length, model depth, and numerical precision. This “Goldilocks zone” clarifies why certain base choices or scheduling strategies succeed in practice while others fail, and why improvements observed under one precision regime may not transfer to another.

Our goal is not to propose a new positional encoding or heuristic modification, but to provide a first-principles characterization of the constraints that govern existing RoPE-based architectures. By grounding RoPE in a signal-processing framework, we aim to clarify the trade-offs underlying long-context transformer design and to offer principled guidance for base selection and scheduling in deep, long-context models.

\subsection{Contributions}
This paper makes the following contributions:
\begin{enumerate}
    \item \textbf{A signal-processing formulation of RoPE as phase modulation.}
    
    We reinterpret rotary positional embeddings as phase modulation applied to a bank of complex oscillators with geometrically spaced frequencies. This formulation makes explicit the spectral structure of RoPE and enables analysis using classical tools from signal processing, providing a unified lens for understanding positional encoding behavior in transformers.

    \item \textbf{Principled lower bounds on the RoPE base from positional stability.}
    
    Under the phase modulation perspective, we derive necessary lower bounds on the RoPE base parameter required to preserve positional coherence over a target context length. These include:
    \begin{itemize}
        \item a fundamental aliasing bound, analogous to a Nyquist limit, which constrains the maximum unambiguous context length, and
        \item a DC-component stability bound that limits phase drift in the lowest-frequency positional modes responsible for long-range alignment.
    \end{itemize}

    \item \textbf{Extension of stability bounds to deep transformers via layer compounding.}
    We show that in multi-layer transformers, repeated rotary modulation compounds angular misalignment across layers, tightening the base requirement as model depth increases. This yields a depth-dependent stability bound that explains why deeper models require larger bases or spectrum-shaping strategies to maintain long-range coherence.

    \item \textbf{An upper bound on the RoPE base induced by numerical precision.}
    We derive a complementary precision-dependent upper bound on the RoPE base arising from finite floating-point resolution. Beyond this limit, incremental phase updates become numerically indistinguishable, leading to positional erasure even in the absence of aliasing. This establishes a hardware-dependent constraint that interacts non-trivially with depth-driven lower bounds.

    \item \textbf{A feasibility region for long-context transformers.}
    Combining the derived lower and upper bounds, we characterize a precision- and depth-dependent feasibility region for RoPE base selection. This “Goldilocks zone” formalizes the trade-off between context length, model depth, and numerical precision, and provides principled guidance for selecting or scheduling RoPE bases in long-context transformer architectures.
    \item \textbf{Validation through case studies of state-of-the-art long-context models.}
    We apply the proposed framework to a range of widely deployed models, including LLaMA, Mistral, and DeepSeek variants, showing that observed successes, failures, and community-driven retrofits align closely with the predicted bounds. These case studies demonstrate that RoPE base selection constitutes a fundamental necessary architectural constraint, and explain previously reported attention collapse, “lost-in-the-middle” behavior, and precision-induced failure modes in long-context transformers.

\end{enumerate}

\subsection{Relation to prior work}
Positional encoding has been extensively studied as a core component of transformer architectures. Early approaches employed absolute positional embeddings learned jointly with token representations, while subsequent work introduced deterministic encodings based on sinusoidal functions to enable length extrapolation. Rotary positional embeddings (RoPE), originally proposed in RoFormer \citep{su2021roformer}, extend this line of work by applying position-dependent rotations directly in the query and key spaces, and have since been adopted in many large-scale language models.

\textbf{Analyses and modifications of RoPE.}\\ 
Recent work such as LongRoPE proposes frequency reparameterization of RoPE embeddings to improve long-context extrapolation by stretching or compressing rotary frequencies across positions \citep{ding2024longrope}. From the perspective of our phase-modulation framework, LongRoPE can be viewed as a learned or adaptive mapping of the nominal position variable, analogous to YaRN \citep{peng2023yarn}, but optimized end-to-end rather than manually designed.

\textbf{Relation to NTK-RoPE and RoPE Scaling Methods.}\\ 
Several recent works have proposed NTK-based scaling rules for rotary positional embeddings, often referred to as NTK-RoPE, which aim to preserve attention kernel similarity when extrapolating beyond the training context length \citep{bloc97ntk2023, chen2023pi, peng2023yarn}. 

Our work is complementary but addresses a fundamentally different question. NTK-RoPE assumes that the underlying positional encoding remains well-defined and numerically stable, and focuses on maintaining functional generalization under extrapolation. In contrast, we analyze when RoPE itself can preserve positional information at all, independent of training dynamics or kernel alignment.

From a signal-processing perspective, NTK-RoPE modifies the effective frequency schedule of RoPE, but does not alter the core phase modulation mechanism or its interaction with depth and numerical precision. As a result, NTK-based scaling operates within the feasibility region defined by our stability and precision bounds, but cannot compensate for violations of these bounds. When the RoPE base lies below the depth-dependent stability minimum, or above the precision-induced upper limit, positional signals necessarily degrade or collapse, regardless of NTK-inspired rescaling.

This distinction helps explain empirical observations in which NTK-style scaling improves extrapolation for some models, yet fails for others with deeper architectures or extreme context lengths. Attention-level empirical analyses similarly observe that RoPE scaling modifies attention weight distributions and long-range interaction patterns rather than preserving the original positional geometry \citep{zhong2025ropeattention}, supporting our interpretation that such techniques delay—but do not eliminate—stability violations. Our framework provides the missing structural constraints that determine whether such scaling strategies are viable in the first place.

\textbf{Connections to sinusoidal and spectral positional encodings.}\\
Sinusoidal positional embeddings and related Fourier-based encodings have long been analyzed through a frequency-domain lens, emphasizing expressivity, extrapolation, and inductive bias \citep{vaswani2017attention, shaw2018relative, press2022alibi, kazemnejad2023positional, tancik2020fourier}. In these approaches, positional information is added linearly to token embeddings, enabling certain spectral interpretations and predictions about long-sequence generalization. Prior work has also shown that self-attention layers can be interpreted as structured convolutional operators, which helps justify frequency-domain analysis of transformers \citep{cordonnier2020relationship}. By contrast, RoPE injects positional information multiplicatively via complex rotations, rather than additively through embedding vectors \citep{su2021roformer, men2024ropebase}. As a result, prior analyses of sinusoidal encodings do not fully capture RoPE’s behavior, particularly under repeated application across multiple transformer layers, where phase compounding and DC-component stability become critical for long-context coherence.

\textbf{Depth, compounding effects, and stability.}\\ 
The interaction between positional encoding and transformer depth has received relatively little theoretical attention. Existing analyses typically consider single-layer or shallow settings, implicitly assuming that positional distortions do not accumulate across layers. In contrast, we explicitly model the compounding effect of rotary modulation across depth, showing that even small per-layer angular deviations can lead to exponential decay of alignment in deep transformers \citep{su2021roformer, peng2023yarn}.

\textbf{Numerical precision and positional encoding.}\\ 
Finite-precision effects in transformers have been studied primarily in the context of training stability and optimization dynamics. However, the impact of floating-point precision on positional encoding fidelity—particularly for large rotation bases and long context lengths—remains underexplored. Our work formalizes this interaction by deriving a precision-dependent upper bound on the RoPE base, revealing a failure mode that is as fundamental as aliasing or architectural design \citep{goldberg1991floating, higham2002accuracy, micikevicius2018mixedprecision}.

\textbf{Positioning of this work.}\\ 
In summary, prior work on RoPE and long-context transformers has largely focused on empirical fixes or geometric reinterpretations. This paper complements these efforts by providing a first-principles, signal-processing-based analysis of RoPE that yields explicit, depth- and precision-aware bounds on the base parameter. Rather than proposing a new positional encoding, we aim to clarify the fundamental constraints that govern existing RoPE-based architectures.

\section{Theoretical background}
\subsection{Rotary Positional Embeddings}
Consider a transformer with hidden dimension d, where rotary positional embeddings are applied to the query ($Q$) and key ($K$) vector representations. Let $x \in \mathbb{R}^d = [x_1, x_2, \dots, x_d]$ denote a token representation before positional encoding. RoPE \citep{su2021roformer} partitions the feature dimension into $d/2$ two-dimensional subspaces $x = [(x_{2i-1}, x_{2i})]$ for $i \in \{1, \dots, d/2\}$  and applies a position-dependent rotation to each subspace as below:
\[ \begin{bmatrix} x'_{2i-1} \\ x'_{2i} \end{bmatrix} = \begin{bmatrix} \cos(p\theta_i) & -\sin(p\theta_i) \\ \sin(p\theta_i) & \cos(p\theta_i) \end{bmatrix} \begin{bmatrix} x_{2i-1} \\ x_{2i} \end{bmatrix} \]

where, $\theta_i = \text{base}^{-2(i-1)/d}$ for $i \in \{1, \dots, d/2\}$ is a frequency specific to the dimension pair, and $p$ is the token position.
Applying these rotations across all subspaces yields the RoPE-transformed representation $\text{RoPE}(x,p)$. This construction ensures that each dimension pair is rotated in a circle by an angle proportional to the token’s position, and dot products $Q \cdot K$ in attention corresponds to position differences rather than absolute positions between tokens.

\subsection{Complex-Valued Representation}
An equivalent and more analytically convenient formulation expresses RoPE using complex-valued features. Each two-dimensional vector pair subspace is identified with a complex number:
\[ z_i = x_{2i-1} + j x_{2i}, \]
where $j = \sqrt{-1}$. Under this representation, the RoPE transformation becomes:
\[ z'_i(p) = z_i \cdot e^{j p \theta_i}. \]
and 
\[
x'_{2i-1} = \Re(z'_i(p)) 
= x_{2i-1}\cos(p\theta_i) - x_{2i}\sin(p\theta_i)
\]
\[
x'_{2i} = \Im(z'_i(p)) 
= x_{2i-1}\sin(p\theta_i) + x_{2i}\cos(p\theta_i)
\]
Thus, RoPE can be viewed as multiplying each complex feature $z_i$ by a position-dependent complex exponential $e^{i \theta_i p}$. This formulation makes explicit that RoPE performs phase modulation with frequency $\theta_i$ and phase $\phi_i(p) = p \theta_i$. RoPE rotates the 2D vector, which is equivalent to a linear phase modulation of a complex vector. 

\subsection{RoPE as Phase Modulation of Oscillator Banks}
Under the complex formulation, RoPE corresponds to a bank of complex oscillators with geometrically spaced frequencies $\mathbf{v}(p) = [e^{j\theta_1 p},\, e^{j\theta_2 p},\, \ldots,\, e^{j\theta_{d/2} p}]$
, Where the angular frequency $\theta_i = \text{base}^{-2(i-1)/d}, \quad i \in \{1, 2, \ldots, d/2\}$. Token position p acts as the modulation variable, advancing the phase of each oscillator linearly with p.
From this perspective, positional encoding fidelity depends on maintaining coherent phase relationships across oscillators. High-frequency components $\theta_{\max} = 1$ (high-frequency, local texture) encode local positional distinctions, while low-frequency components $\theta_{\min} = \frac{1}{\text{base}}$ (low-frequency, global context) dominate long-range alignment. In particular, the lowest-frequency modes play a critical role in preserving coarse positional structure over long contexts.
This oscillator-bank view connects RoPE directly to classical signal processing concepts, such as aliasing, phase drift, and frequency resolution, which form the basis of our subsequent theoretical analysis.

\subsection{Depth and Repeated Rotary Modulation}
In a transformer with $N$ layers, RoPE is applied independently at each layer to the query and key representations. Let $z_i^{(\ell)}$ denote the complex feature at layer $\ell$. The effect of RoPE across layers can be expressed as
\[
z_i^{(\ell)}(p) = z_i^{(\ell)} e^{j p \theta_i} \quad \ell \in \{1, 2, \ldots, N\}.
\]

Although the rotation angles are identical across layers, their effect on positional alignment compounds through repeated application. Small angular misalignments introduced at each layer—whether due to frequency choice, depth-dependent interactions, or numerical error—can accumulate, leading to decay in positional coherence as depth increases.
This observation motivates our depth-dependent stability analysis in later sections.

\subsection{Numerical Precision Considerations}
In practical implementations, RoPE is computed using finite-precision arithmetic, typically in FP16, BF16, or FP32. For large RoPE bases or long context lengths, the incremental phase $p \theta_i$ may fall below machine precision, causing $e^{j p \theta_i}$ to become numerically indistinguishable from unity. The behavior of floating-point rounding and resolution in such regimes follows classical results in numerical analysis \citep{goldberg1991floating, higham2002accuracy}, which characterize how representable spacing grows with magnitude. This connection strengthens that RoPE phase increments are fundamentally limited by well-established numerical bounds.

This phenomenon effectively erases positional information for affected frequency bands, even when aliasing constraints are satisfied. Understanding this limitation requires explicit consideration of floating-point resolution, which we incorporate into our analysis when deriving precision-dependent upper bounds on the RoPE base.

\subsection{Summary and Setup for Analysis}
This section establishes a signal-processing interpretation of RoPE as phase modulation over a bank of complex oscillators. This formulation exposes the key quantities—frequency spacing, phase accumulation, depth compounding, and numerical resolution—that govern positional encoding stability. In the following sections, we leverage this perspective to derive explicit lower and upper bounds on the RoPE base parameter for long-context transformers.

\section{Theoretical bound for RoPE base}
In this section, we derive necessary lower bounds on the RoPE base parameter base required to preserve positional coherence over a target context length. These bounds arise from fundamental constraints on phase modulation and hold independently of architectural details or training procedures. We begin with a frequency-domain aliasing constraint and then derive a stability bound for low-frequency (DC-aligned) modes. Finally, we extend the analysis to deep transformers, showing how these constraints tighten with depth.

\subsection{Preliminaries and Notation}
Let $p \in \{0, 1, \ldots, L-1\}$ denote token position, where $L$ is the target context length. Recall that RoPE applies a phase rotation
$z'_i(p) = z_i \cdot e^{j p \theta_i}$, where $\theta_i = \text{base}^{-2(i-1)/d}$ for $i \in \{1, 2, \ldots, d/2\}$. We say that positional encoding is \emph{coherent} over length $L$ if distinct positions induce distinguishable phase configurations across the oscillator bank, and if relative positional alignment is preserved under attention.

\subsection{Theoretical lower bound for RoPE base}
\subsubsection{The Fundamental Aliasing Limit}
For the model to distinguish unique positions across a context window $L$, the positional encoding function must be injective (one-to-one) over the domain [0, $L$]. While high-frequency heads wrap around (alias) repeatedly, the model resolves this ambiguity using lower-frequency heads (similar to the Chinese Remainder Theorem). However, the lowest frequency head at  $\theta_{\min}$ (the "Fundamental") cannot be aliased. It serves as the absolute global reference.

The period (wavelength) of the fundamental oscillator is:
\[
T_{\text{fund}} = \frac{2\pi}{\theta_{\min}} = 2\pi \cdot \text{base}
\]
\textbf{Theorem 1 (The Fundamental Aliasing Limit)}: 
To preserve the global context stability, the context length $L$ must fall within the primary period of the fundamental oscillator, i.e.,
\[L < T_{\text{fund}} = 2\pi \cdot \text{base}\]
Which is:  $\text{base} > \frac{L}{2\pi}$

This inequality is the Nyquist Stability Bound \citep{shannon1949communication, unser2000sampling} for RoPE. If the base violates this, the global position signal becomes ambiguous (aliases), and the "grid" collapses.

\begin{figure}[h]
\begin{center}
\includegraphics[width=0.9\linewidth]{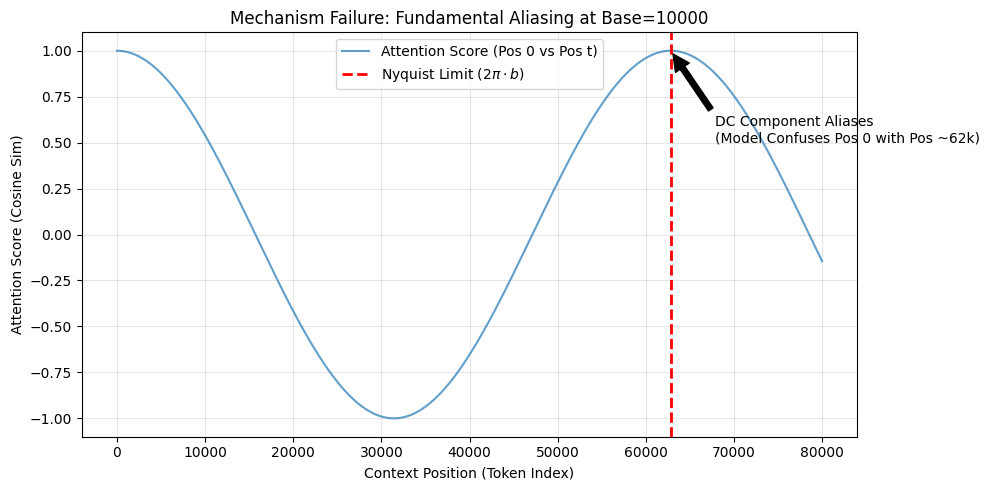}
\end{center}
\caption{Fundamental Aliasing in RoPE Embeddings (Base=10,000).}
\end{figure}

To expose the periodicity vulnerability in RoPE, we generate rotary embeddings for a sequence of 100,000 tokens. We extract the query vector at Position 0 and compute its cosine similarity (attention score) against every key position from 0 to 80,000.

The results confirm the aliasing hypothesis. The attention score follows a smooth cosine decay from 1.0 (perfect self-attention at Position 0) down to -1.0, then back up—until it spikes back to near-perfect correlation at Position $\approx 62,832$.  This position corresponds exactly to $2\pi \cdot b$ where $b=10,000$ (the standard RoPE base). At this Nyquist-like limit, this creates a ``collision horizon'' at $\sim 62$k tokens. The model fundamentally cannot distinguish Position 0 from Position $\sim 62$k. Beyond this point, the model's positional understanding breaks down not gradually, but catastrophically—confusing the beginning of context with arbitrary distant positions. This explains why naive context window extension without base adjustment leads to incoherent long-range attention.

\subsubsection{Single layer DC-component stability Limit} 
In classical signal processing, the DC component refers to the zero-frequency component of a signal, corresponding to its mean value and governing global signal structure \citep{oppenheim2014dsp}. Rotary positional embeddings do not contain a literal zero-frequency term; however, the lowest-frequency RoPE mode plays an analogous role as a quasi-DC reference component. Over long context lengths, this mode evolves most slowly with position and therefore governs the global positional coordinate system used by the attention mechanism.

Under RoPE, positional information is encoded through phase rotations applied to complex-valued feature pairs. For a frequency index i, the phase evolves as a linear function of token position. The lowest-frequency component accumulates phase most gradually, enabling it to preserve coarse positional ordering across distant tokens. Stability of this component is therefore critical: excessive phase drift in this mode disrupts the global positional reference frame even if higher-frequency components remain locally coherent.

This behavior parallels classical phase modulation systems, where a constant bias in the modulating signal produces only a fixed phase offset rather than an accumulating distortion \citep{haykin2009communication, proakis2002cse, carlson2010communication}. In such systems, global phase offsets preserve relative signal geometry, whereas position-dependent phase errors introduce cumulative misalignment. By analogy, RoPE remains invariant to global phase rotations but becomes unstable when per-position phase increments grow too large.

From this perspective, maintaining stability of the quasi-DC RoPE component requires that successive positional phase increments remain sufficiently small to preserve high cosine similarity between adjacent rotations. When this condition is violated, even modest angular deviations accumulate with position and systematically attenuate the contribution of the lowest-frequency component to attention alignment.

\textbf{Theorem 2 (DC-component stability Limit)}:  To maintain stable long-range positional alignment over a context length $L$, the RoPE base must satisfy
\[ \text{base} \ge \frac{L}{\arccos(\varepsilon)} \]
Where $L$ is the maximum sequence length you want the model to handle and $\epsilon$ is minimum acceptable correlation, called coherence (e.g., 0.95), so that attention still works.  

\begin{figure}[h]
\begin{center}
\includegraphics[width=0.9\linewidth]{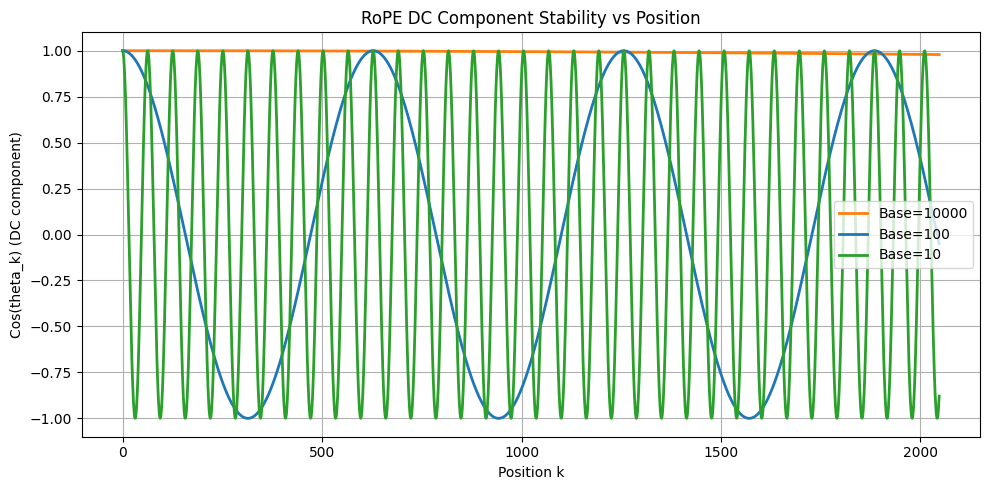}
\end{center}
\caption{Effect of RoPE Base on DC component Component Stability.}
\end{figure}

Figure 2 plots $cos(p \theta_i)$ across positions 0–2000, comparing three RoPE bases: 10,000 (standard), 100, and 10. Only Base=10,000 (yellow flat line at 1.0) maintains stable phase coherence. The smaller bases exhibit severe oscillation—Base=100 cycles multiple times, while Base=10 oscillates so rapidly it appears as high-frequency noise.
These observations are consistent with Theorem 2. A base that's too small causes the DC component to "spin" too fast: the model loses its global positional reference frame because distant positions have uncorrelated phase values. Even at position 2000, Base=10 has completed ~60+ full rotations. The model cannot distinguish position 0 from position 100, 200, etc.

\textbf{Representative Coherence Regimes Induced by the Stability Threshold $\varepsilon$}

The stability bounds derived above depend on a cosine similarity threshold $\varepsilon$, which specifies the minimum acceptable preservation of phase alignment for the lowest-frequency RoPE component. In this framework, $\varepsilon$ controls the allowable angular deviation per positional increment and therefore determines the trade-off between positional coherence and required RoPE base.

Because the minimum base scales as 1/arccos($\varepsilon$), different values of $\varepsilon$ correspond to distinct coherence regimes. While the optimal threshold may depend on downstream tasks and architectural factors (Future work may explore task-dependent selection of $\varepsilon$ through empirical calibration or adaptive stability criteria), several representative operating points illustrate the design trade-offs implied by the theory:

\paragraph{Moderate coherence regime ($\varepsilon \approx 0.9$).}
This regime permits modest phase drift while maintaining stable long-range alignment. It corresponds to a multiplier of $\frac{1}{\arccos(0.9)} = 2.22$. on the effective context length, providing a balanced trade-off between stability and parameter growth.

\paragraph{Strong coherence regime ($\varepsilon \approx 0.95$).}
This setting significantly reduces allowable phase drift and improves robustness of long-range dependencies, particularly in reasoning-intensive workloads. It produces a multiplier of $\frac{1}{\arccos(0.95)} = 3.15$.

\paragraph{High-fidelity coherence regime ($\varepsilon \in [0.99, 0.995]$).}
This regime enforces near-perfect phase preservation and may be appropriate for highly noise-sensitive or extremely long-context scenarios. However, it dramatically increases the required base, with multipliers ranging approximately from 7 to 10. 

Intuitively, when the RoPE base falls below the stability bound implied by a chosen $\varepsilon$, phase rotations accumulate too rapidly with position. This causes the lowest-frequency (DC-aligned) component to decay, leading to violent oscillations in the attention kernel and degradation of long-range attention. When the base exceeds the bound, phase accumulation is sufficiently slow to preserve low-frequency coherence, enabling stable attention over the full context length.

Importantly, these regimes arise directly from the geometric relationship between cosine similarity and phase deviation, rather than empirical tuning. They illustrate how tightening positional coherence requirements rapidly increases the minimum base required for stability. Also, this stability condition captures a failure mode not explained by aliasing alone. Even when positional indices remain theoretically distinguishable, excessive low-frequency phase drift can erode alignment and suppress attention scores. Enforcing a minimum cosine similarity threshold ensures that phase modulation preserves the effective DC-aligned signal and prevents the gradual loss of long-range positional coherence.

\subsubsection{Depth-Induced Compounding of DC-component stability Limit} 
We now extend the stability analysis to deep transformers. RoPE is applied independently at each layer, and while the nominal rotation angles are identical, their effect on positional alignment compounds multiplicatively.  Assuming that lowest per-layer cosine similarity for the lowest-frequency component is $\rho$
\[ \rho = \cos\!\left(\frac{L}{\text{base}}\right) \]

and that these effects are approximately independent across layers (similar independence approximations are commonly used in deep signal propagation and neural network stability analyses) \citep{poole2016chaos, xiong2020layernorm, liu2020difficulty}, the cumulative similarity after $N$ layers becomes:
\[\rho_{\text{total}} = \rho^N \]

To ensure that the total similarity remains above a global coherence threshold $\varepsilon$, we must have:
\[\rho_{\text{total}} = \rho^N \ge \epsilon \Rightarrow \rho \ge \epsilon^{1/N}\]

Substitute this into the original bound in Theorem 2, and we get
\[\text{base} \;\ge\; \frac{L}{\arccos\!\left(\epsilon^{1/N}\right)}\]

\textbf{Theorem 3 (Compounding multi-layer DC-component stability Limit)}:  In a transformer with $N$ layers, maintaining stable positional coherence $\epsilon$ over context length $L$ requires the RoPE base:
\[\text{base} \;\ge\; \frac{L}{\arccos\!\left(\epsilon^{1/N}\right)}\]

Intuitively, this bound grows faster than the single-layer DC bound as $N$ increases. For deep models (e.g., $N$ = 48), even modest $\epsilon$ = 0.9 leads to a much tighter per-layer requirement. This explains why deeper models need either larger base, or spectrum shaping (e.g., NTK-aware RoPE, YaRN) to preserve long-range coherence without inflating base.

\begin{figure}[h]
\begin{center}
\includegraphics[width=0.9\linewidth]{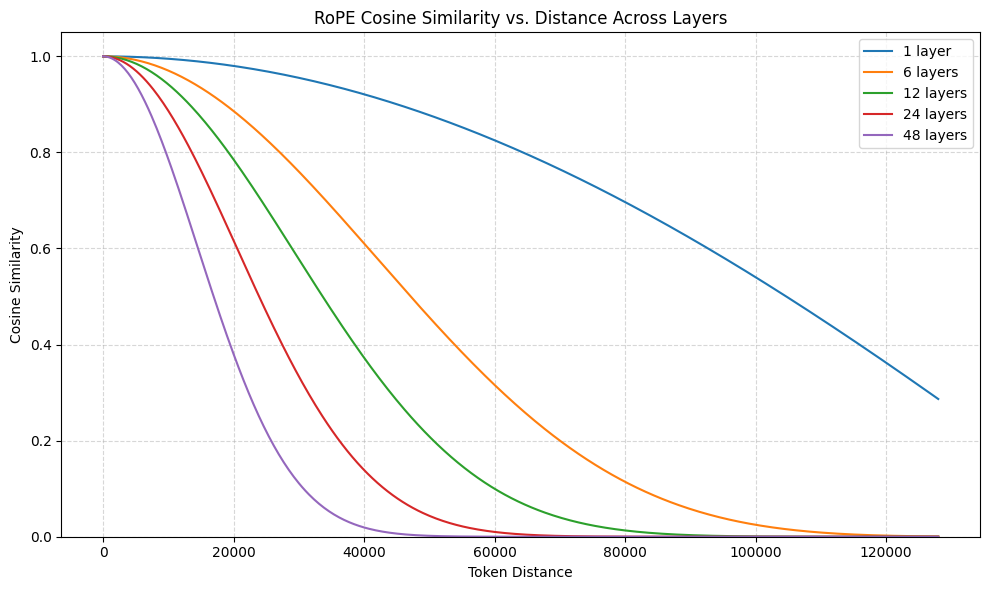}
\end{center}
\caption{Effect of RoPE Base on DC component Component Stability.}
\end{figure}

Figure 3 simulates Theorem 3 empirically: cosine similarity decays progressively faster with depth. A single layer retains $\sim 30\%$ similarity at 120k tokens; 6 layers hold to $\sim$110k; and 48 layers collapse to near-zero by merely 50k. This is the compounding bound in action—each layer multiplies rotation error, so deep models amplify small phase deviations into catastrophic long-range incoherence. Shallow stacks tolerate loose phase constraints; deep architectures demand exponentially stricter stability.

\textbf{Theorem 4 (theoretical lower bound on the RoPE base)}: The theoretical lower bound on the RoPE base is essentially the maximum of The Fundamental Aliasing Limit and Compounding multi-layer DC-component stability Limit, i.e.,
\[ \text{base}_{\min} = \max\!\left( \frac{L}{2\pi}, \frac{L}{\arccos\!\left(\epsilon^{1/N}\right)} \right) \]

it guarantees the model doesn’t lose coherent signal at long distances — the low-frequency part of the positional signal remains stable, and no fundamental aliasing.

\subsubsection{Discussion}
The bounds derived in this section are necessary but not sufficient conditions for long-context stability. They arise from fundamental properties of phase modulation and hold regardless of training data, optimization strategy, or architectural refinements.
Crucially, these results imply that extending context length without adjusting the RoPE base—or without reshaping the frequency spectrum—will inevitably lead to positional degradation, particularly in deep models. In the next section, we show that increasing the base indefinitely is also impossible due to finite numerical precision, leading to a complementary upper bound.

\subsubsection{Connection with YaRN’s frequency spectrum}
YaRN \citep{peng2023yarn} introduces a position-dependent rescaling function $f(p)$ that compresses large positional indices, reducing the effective rotation angle applied at long distances. Under YaRN, rotary embeddings no longer apply a phase rotation proportional to the raw position $p$, but instead rotate by $f(p)$, yielding an effective phase of $f(p) \theta_k$ for frequency index $\theta_k$.
From the perspective of our phase modulation framework, YaRN can be interpreted as a reparameterization of position, rather than a modification of the underlying RoPE mechanism. The spectral structure of RoPE and its layer-wise compounding behavior remain unchanged; only the maximum effective positional extent is altered.
As a result, the stability bounds derived in Sections 3 apply directly after replacing the nominal context length $L$ with the effective length $f(L)$. In particular, the DC-component stability condition becomes

\[ \text{base}_{\min, \text{YaRN}} = \max\!\left( \frac{f(L)}{2\pi}, \frac{f(L)}{\arccos\!\left(\epsilon^{1/N}\right)} \right) \]
where $N$ denotes the number of transformer layers.
This reformulation clarifies why YaRN-style scaling can extend usable context lengths for some models: by reducing $f(L)$, it relaxes the lower bound on the required base. However, YaRN does not alter the fundamental trade-off between base, depth, and numerical precision. If the resulting effective base remains below the stability threshold, or exceeds the precision-induced upper bound, positional degradation and attention collapse remain unavoidable.

\subsection{Theoretical upper bound for RoPE base from Finite Precision}
In the previous section, we derived necessary lower bounds on the RoPE base required to preserve positional coherence over long contexts. One might therefore attempt to increase the base indefinitely as context length grows. In practice, however, this strategy encounters a fundamental limitation imposed by finite numerical precision. The physical implementation of RoPE on digital hardware introduces a competing Upper Bound.

\textbf{The Smearing Effect}: In floating-point arithmetic (e.g., FP32), the gap between representable numbers increases as the magnitude of the value grows \citep{higham2002accuracy, goldberg1991floating}. For a position $p$ and a phase increment 
$\Delta\theta = \text{base}^{-2i/d}$, the rotation "smears" or "erases" when the increment falls below the hardware's resolution at that specific magnitude.

\textbf{Theorem 5 (Numerical Erasure)}: Positional coherence is maintained only if the phase increment $\Delta\theta$ is larger than the machine epsilon $\epsilon_{\text{mach}}$ (approximately $7.81 \times 10^{-3}$, $9.77 \times 10^{-4}$, $1.19 \times 10^{-7}$, $2.22 \times 10^{-16}$ for BF16, FP16, FP32 and FP64 respectively). The RoPE base requires:
\[\text{base} < 1/\epsilon_{\text{mach}}\]
This bound arises from finite-precision arithmetic \citep{higham2002accuracy, goldberg1991floating}. Increasing the RoPE base beyond this limit, the fundamental oscillator enters a state of Phase Erasure. Even if the model hasn't reached the Nyquist alias point, the ``ruler'' stops moving because $p+\Delta\theta$ rounds back to $p$ in the hardware's mantissa. This creates ``dead zones'' in long-context attention where the model cannot distinguish between tokens $p$ and $p+1$, effectively breaking the relative positional encoding property of RoPE.

\subsection{The Feasibility Region for RoPE Base Selection}
Combining the lower bounds from Theorem 4 with the precision-induced upper bound in Theorem 5 yields a constrained feasibility region:
\[
\max\!\left(
    \frac{L}{2\pi},
    \frac{L}{\arccos\!\left(\epsilon^{1/N}\right)}
\right)
\;<\;
\text{base}
\;<\;
\frac{1}{\varepsilon_{\text{mach}}}
\]

We refer to this interval as the feasibility region for RoPE base selection. When the interval is non-empty, long-context positional coherence is theoretically achievable. When the interval collapses, no choice of base can simultaneously satisfy stability and precision constraints.

This characterization explains several empirical observations: (i) why increasing the base improves long-context behavior only up to a point; (ii) why deeper models require more careful base selection, and (iii) why strategies effective under FP32 may fail under lower-precision arithmetic.

\subsection{Discussion}
The upper bound highlights a fundamental limitation of RoPE that complements aliasing and architectural design. Unlike lower bounds, which scale with context length and depth, the precision-induced upper bound is dictated by hardware and numerical format. Together with the lower bounds, this result formalizes a “Goldilocks zone” for RoPE base selection and motivates spectrum-shaping or scheduling strategies that redistribute positional capacity without exceeding precision limits. 

\section{Case Study: RoPE Bounds in State-of-the-Art Transformer Models}
To evaluate the practical relevance of the proposed RoPE base bounds, we analyze a set of widely used state-of-the-art transformer models, comparing their deployed RoPE configurations against both existing empirical heuristics and the theoretical limits derived in Sections 3. This section grounds the theory in real architectural choices made by model designers, highlighting where current systems possibly operate safely—and where they possibly approach or violate fundamental stability constraints.

\subsection{Methodology}
For each model, we collect the number of layers ($N$), maximum context length ($L$) and actual or effective RoPE base used in deployment.  We then compute Baichuan empirical lower bound \citep{men2024ropebase}, theoretical minimum base from Theorem 4 (multi-layer DC-component stability and aliasing avoidance), using a coherence threshold $\epsilon$ = 0.95, and theoretical maximum base from Theorem 5 (finite-precision numerical erasure), assuming FP32 arithmetic ($\varepsilon_{\text{mach}} = 1.19 \times 10^{-7}$ for FP32). Each model is classified according to whether its deployed base lies within the Goldilocks Zone—i.e., strictly between the theoretical minimum and maximum bounds.

\subsection{Comprehensive RoPE Bound Comparison (\texorpdfstring{$\epsilon$ = 0.95}{eps = 0.95})}

\begin{table*}[t]
\centering
\footnotesize

\resizebox{\textwidth}{!}{%
\begin{tabular}{l c c c c c c c c l}
\toprule
\textbf{Model} &
\textbf{Layers} &
\textbf{Context} &
\textbf{Actual Base} &
\textbf{Baichuan LB} &
\textbf{Aliasing} &
\textbf{DC Stability} &
\textbf{Min Base} &
\textbf{Max Base} &
\textbf{Status} \\
\midrule
\\ \hline \\
LLaMA-7B & 32 & 2k & 10{,}000 & 16{,}000 & 326 & 36{,}197 &
$>36{,}197$ & $<8{,}388{,}608$ & {\color{orange}Unstable} \\

LLaMA2-7B & 32 & 4k & 10{,}000 & 27{,}000 & 652 & 72{,}394 &
$>72{,}394$ & $<8{,}388{,}608$ & {\color{orange}Unstable} \\

Baichuan2-7B & 32 & 4k & 10{,}000 & 27{,}000 & 652 & 72{,}394 &
$>72{,}394$ & $<8{,}388{,}608$ & {\color{orange}Unstable} \\

LLaMA3-8B & 32 & 8k & 500{,}000 & 84{,}000 & 1304 & 144{,}785 &
$>144{,}785$ & $<8{,}388{,}608$ & {\color{green}Very Stable} \\

Mistral-v0.2 & 32 & 32k & $10^6$ & 640{,}000 & 5215 & 579{,}145 &
$>579{,}145$ & $<8{,}388{,}608$ & {\color{green}Very Stable} \\

DeepSeek-V2 (YaRN=40) & 60 & 128k & $4\times 10^5$ & 7{,}800{,}000 & 20{,}859 &
3{,}170{,}586 & $>3{,}170{,}586$ & $<8{,}388{,}608$ &
{\color{orange}Unstable} \\

DeepSeek-V3 (YaRN=40) & 61 & 128k & $6.4\times 10^6$ & 7{,}800{,}000 & 20{,}859 &
3{,}195{,}319 & $>3{,}195{,}319$ & $<8{,}388{,}608$ &
{\color{green}Very Stable} \\

Target (Hypothetical) & 96 & 1M & --- & $5.1\times 10^8$ & 166{,}894 &
32{,}086{,}169 & $>32{,}086{,}169$ & $<8{,}388{,}608$ &
{\color{red}Infeasible} \\

\bottomrule
\end{tabular}
} 
\caption{Stability bounds for various transformer models under FP32 precision.}
\begin{tablenotes}
\footnotesize
\item(*The base for DeepSeek-v2 and v3 is reported or inferred from publicly available configurations, incorporates YaRN, and is approximated for comparison purposes.)
\end{tablenotes}
\end{table*}

\subsection{Key Observations}
\subsubsection{Explaining “Marginal but Operational” Models}
Early LLaMA-family models (e.g., LLaMA-7B and LLaMA-2-7B) violate both the Baichuan empirical lower bound and the theoretical DC-stability minimum derived in Section 3. Under our framework, these configurations lie within a marginal or weakly unstable regime, where layer-compounded DC-component decay is expected but does not immediately produce catastrophic failure at moderate sequence lengths.

This apparent discrepancy arises because empirical evaluations often measure average downstream performance rather than positional signal integrity. Our analysis indicates that moderate-depth models can remain operational despite sub-threshold RoPE bases, as positional degradation accumulates gradually and may remain partially masked by training priors, local attention structure, and dataset statistics. However, such models exhibit reduced positional robustness, making them significantly more sensitive to long-range dependency tasks and context extrapolation.

This interpretation is consistent with later empirical observations showing that early LLaMA variants exhibit systematic degradation when sequence length approaches or exceeds their training context, despite appearing functional under standard benchmarks. Rather than contradicting empirical heuristics, these results suggest that empirical bounds capture coarse operational viability, whereas the theoretical stability bound predicts long-range positional reliability and failure onset.

\subsubsection{Diagnosing Failures and Base-Driven Fixes in Long-Context Models}
Several widely deployed long-context models exhibit systematic degradation that cannot be explained by empirical RoPE heuristics alone. In this section, we analyze representative failure cases and show that their observed behavior is consistent with violations of the theoretical stability bound derived in Section 3.

\textbf{LLaMA-2-7B: Stability Deficit at Native Context Length}

LLaMA-2-7B employs a RoPE base of 10{,}000 with 32 transformer layers and a nominal context length of 4k tokens. Under our stability analysis, this configuration lies below the minimum base required to preserve the DC-aligned positional component across depth ($\sim 72{,}394$ for 32 layers at a 4k-token context), even before any context extension is attempted.

Empirically, this instability manifests in two well-documented ways \citep{liu2024lost}. First, in the Lost-in-the-Middle study, LLaMA-2-7B performance dropped sharply when relevant information appeared near the middle of the input sequence. In our framework, this behavior aligns with progressive attenuation of the global positional signal, as small per-layer phase misalignments compound multiplicatively, weakening positional coherence far from sequence boundaries.

Second, community evaluations report an abrupt increase in perplexity once sequence length exceeds the trained 4k limit \citep{touvron2023llama2, touvron2023llama2, togetherai2023evals}. Rather than degrading smoothly, performance collapses over a narrow range of additional tokens. Similar abrupt transitions in long-range attention structure have been observed in attention-centric analyses of RoPE extensions \citep{zhong2025ropeattention}, which report sharp redistribution of attention mass when scaling strategies operate outside stable positional regimes.

This pattern is characteristic of layer-compounded phase decoherence: while low-frequency RoPE components remain well-behaved, higher-frequency bands alias early, and their effects propagate across layers until a critical instability threshold is crossed. Importantly, this failure does not require hard phase wraparound at a single frequency. Instead, it emerges from the interaction of finite RoPE base, model depth, and accumulated phase error.

\textbf{Effectiveness of Large-Base RoPE Retrofits}

The strongest practical evidence supporting the role of RoPE base selection comes from both community-driven long-context retrofits of LLaMA-2 and first-party architectural choices in LLaMA-3.

Several independent efforts extended LLaMA-2 to 32k context by increasing the RoPE base from 10,000 to approximately 500,000 while leaving model depth and attention structure unchanged \citep{chen2023pi, peng2023yarn, bloc97ntk2023}. According to our theoretical bounds, this modification shifts the model from a weakly unstable regime—where DC-component decay accumulates across layers—to a stable operating region in which multi-layer positional coherence can be preserved.

This observation explains why long-context fine-tuning becomes viable only after increasing the RoPE base, and why attempts to retain the original base consistently fail \citep{liu2024lost, men2024ropebase}. These retrofits demonstrate that RoPE base selection constitutes a fundamental architectural constraint for long-context modeling. Once the base falls below the theoretical stability minimum, positional degradation cannot be fully compensated by optimization, data scaling, or training strategies.

In contrast, LLaMA-3-8B adopts a RoPE base of approximately $500{,}000$ as part of its native design \citep{touvron2023llama2}. This value exceeds both the theoretical stability minimum for an 8k-token context ($\approx 144{,}785$ for 32 layers) and the Baichuan empirical bound ($\approx 84{,}000$). Correspondingly, LLaMA-3 exhibits significantly improved long-context robustness, supporting the claim that depth- and length-aware theoretical bounds provide a more reliable predictor of stability than fixed empirical heuristics.

\textbf{DeepSeek-V2: Violating Both Empirical and Theoretical Bounds}

DeepSeek-V2 provides a larger-scale counterexample. The model uses 60 transformer layers and supports a 128k context window. Its nominal RoPE base of $10{,}000$ is augmented by YaRN scaling, yielding an effective inference-time base of approximately $4 \times 10^{5}$ \citep{deepseekv2}. However, this effective value remains substantially below both the empirical Baichuan lower bound and the theoretical stability minimum predicted by Theorem~4 ($\approx 3.17 \times 10^{6}$).

Under these conditions, multi-layer DC-component decay and frequency aliasing become unavoidable. As model depth amplifies phase misalignment, long-range positional signals deteriorate, producing unstable attention patterns and degraded coherence over extended contexts. Reported weaknesses in long-context reasoning and retrieval therefore follow naturally from operating outside the theoretically stable regime rather than representing anomalous training failures.

\textbf{Summary}

Across both moderate and long context lengths, these case studies reveal a consistent pattern: when the RoPE base falls below the stability bound determined by model depth and target context length, positional signal degradation becomes unavoidable. While such failures were previously documented empirically, our framework provides a principled explanation for why they arise, why they often emerge abruptly, and why they cannot be reliably mitigated through training or optimization alone.

Similar scaling-driven generalization breakdowns have been observed in transfer learning literature, where models exhibit predictable performance collapse when evaluated outside their training distribution \citep{hernandez2021scaling}. In the context of RoPE-based transformers, the stability bound plays an analogous role by defining the architectural limits of positional extrapolation.

\subsubsection{The Precision Wall at One Million Tokens}
Extending transformers to extremely long contexts introduces a numerical feasibility constraint in addition to positional stability: the finite-precision representation of RoPE phase rotations. Even when the RoPE base satisfies the stability minimum, sufficiently long sequences can exceed the representational resolution of FP32 arithmetic, leading to positional signal attenuation and degraded attention structure \citep{higham2002accuracy, goldberg1991floating}. Related precision-induced degradation has also been observed in long-context attention simulations and mixed-precision training studies \citep{micikevicius2018mixedprecision}.

\paragraph{Empirical and illustrative examples:}

\begin{itemize}

\item \textbf{DeepSeek-V2:} supports a 128k-token context window using 60 transformer layers. The model employs a nominal RoPE base of $10{,}000$ and relies on YaRN scaling, yielding an effective inference-time base of approximately $4 \times 10^{5}$ \citep{deepseekv2}. This effective base remains substantially below both the theoretical stability minimum ($\approx 3.17 \times 10^{6}$) and the practical FP32 precision ceiling ($\approx 8.4 \times 10^{6}$).

Operating in this regime produces two interacting failure modes. First, the base is insufficient to preserve the DC-aligned positional component across depth, resulting in progressive multi-layer phase decoherence. Second, aliasing in higher-frequency RoPE components emerges early in the sequence. Together, these effects lead to degraded long-range attention coherence. Empirical evaluations report significant degradation in retrieval and multi-hop reasoning tasks at long sequence lengths, consistent with these predicted stability violations.

\item \textbf{DeepSeek-V3:} extends the architecture to 61 layers and increases the nominal RoPE base to $160{,}000$, combined with YaRN scaling that produces an effective inference-time base of approximately $6.4 \times 10^{6}$ \citep{deepseekv3}. This configuration exceeds the theoretical stability minimum for a 128k-token context ($\approx 3.20 \times 10^{6}$) while remaining below the FP32 precision ceiling.

Consequently, DeepSeek-V3 operates within the theoretically feasible region but remains close to the numerical precision boundary. While the model maintains improved positional coherence relative to earlier configurations, attention behavior near extreme context lengths exhibits increased sensitivity to positional noise. This observation suggests that inference-time scaling strategies can partially mitigate stability deficits but do not eliminate numerical precision constraints.

\item \textbf{Target 1M Model} (1 million tokens, 96 layers) illustrates the extreme limit of RoPE scaling. Preserving DC-aligned positional components across this depth requires a RoPE base exceeding approximately $3.2 \times 10^{7}$. However, representing phase increments associated with such large bases approaches the mantissa resolution limits of FP32 arithmetic.

According to the upper bound derived in Section 3, FP32 precision imposes an approximate representational ceiling near $\approx 8.4 \times 10^{6}$. When the required base exceeds this value, phase increments between adjacent token positions become smaller than machine resolution, causing positional phase updates to round to identical values. This effect produces what we term the Precision Wall: a regime in which positional distinctions collapse despite theoretical stability of the encoding scheme.

\end{itemize}

\textbf{Dual Constraints on Ultra-Long Context Transformers}

These examples reveal two complementary architectural limits governing RoPE scaling:

\textbf{Stability Bound}: The RoPE base must exceed a minimum value determined by model depth and target context length to preserve DC-aligned positional signals. Violating this bound produces cumulative phase misalignment and multi-layer aliasing.

\textbf{Precision Wall}: Even when the stability bound is satisfied, extremely large RoPE bases interact with finite floating-point resolution. When phase increments fall below machine precision, positional updates become numerically indistinguishable, degrading relative positional encoding.

These results demonstrate that arbitrarily extending transformer context length requires coordinated scaling of RoPE base, model depth, and numerical representation. Increasing context length without accounting for both stability and precision constraints risks predictable degradation in long-range attention behavior. Future ultra-long-context architectures may therefore require either alternative positional parameterizations or higher-precision numerical formats to overcome these limits.

\subsection{Implications and Practical Guidelines}
The analysis of real-world long-context models, retrofits, and extreme-scale experiments highlights several fundamental lessons about RoPE design and deployment:

\textbf{RoPE Base Selection is a Fundamental Architectural Constraint for Stable Long-Context Operation}

Across varying context lengths, model depths, and sequence extension strategies, the RoPE base must satisfy the stability bound derived in Section 3 to preserve DC-aligned positional components. When the base falls below this depth- and length-dependent minimum, multi-layer cosine decay accumulates, leading to aliasing and abrupt attention collapse. Community-driven LLaMA-2 32k retrofit experiments \citep{ding2024longrope, peng2023yarn, bloc97ntk2023} demonstrate that increasing the RoPE base alone—without modifying attention structure or network depth—can shift models from unstable to stable operating regimes. These results indicate that RoPE base selection is not merely a tunable hyperparameter, but a defining architectural factor that governs the feasible operating range of long-context transformers. Once the base lies below the theoretical minimum, stability is unlikely to be reliably recovered through optimization, data scaling, or training heuristics alone.

\textbf{Depth- and Length-Aware Bounds Predict Observed Model Behavior}

Across multiple architectures—including LLaMA-2, LLaMA-3, and DeepSeek families—depth- and context-aware stability bounds consistently predict observed success and failure regimes, while fixed empirical heuristics frequently fail to generalize. Incorporating both model depth and target context length provides a predictive and actionable framework for designing or retrofitting long-context transformers. These bounds enable practitioners to anticipate multi-layer positional degradation before training or deployment, reducing reliance on trial-and-error scaling strategies.

\textbf{Empirical Rules are Insufficient}

Widely adopted heuristics, such as fixed bases of 10k or 16k, fail to account for the coupled interaction between transformer depth and target context length. Empirical failures such as LLaMA-2-7B’s “lost in the middle” phenomenon \citep{liu2024lost} and the long-range collapse observed in DeepSeek-V2 provide practical confirmation of the theoretical predictions derived in this work. Similar limitations of static RoPE scaling have been observed in recent long-context studies \citep{men2024ropebase}. Models that satisfy empirical rules calibrated for shallow architectures or shorter contexts may still experience catastrophic DC-component decay when deployed in deeper networks or extended context regimes.

\textbf{Finite-Precision Limits Impose an Upper Bound}

Even when the RoPE base satisfies the stability minimum, numerical representation introduces an additional constraint for extremely long contexts. The 1M-token target configuration illustrates that bases exceeding approximately 8.4 million induce catastrophic numerical erasure under FP32 precision, independent of training strategy or architectural design. This “Precision Wall” establishes a hard upper bound on positional encoding fidelity in current floating-point representations. Crossing this boundary causes positional distinctions to collapse, resulting in attention degeneracy and loss of long-range dependency modeling. Consequently, long-context scaling cannot be pursued indefinitely without coordinated changes to numerical precision, positional representation, or computational format.

Practical Guidelines for Model Design and Retrofits
The theoretical and empirical results of this study suggest several practical design guidelines:
\vspace{-4pt}
\begin{itemize}
    \item Compute the stability minimum based on layers and target context before training or fine-tuning.
    \item Ensure the RoPE base does not exceed numerical precision limits to prevent the “Precision Wall.”
    \item Validate multi-layer positional coherence through simulation or small-scale experiments.
    \item Prioritize base adjustment in retrofits of long-context models rather than relying on architectural hacks, as attention structure alone cannot recover lost DC coherence.
\end{itemize}
\vspace{-4pt}

In summary, RoPE base selection, model depth, and numerical precision must be jointly optimized for reliable long-context transformer performance. Following these principles ensures stable multi-layer positional coherence, prevents catastrophic attention collapse, and provides actionable guidance for both new model designs and community-driven long-context extensions.

\section{Conclusion and Future Work}
In this work, we provide a theoretical framework for understanding Rotary Positional Embeddings (RoPE) in long-context transformers, quantifying both stability lower bounds and finite-precision upper bounds. Key takeaways include:

\begin{itemize}
    \item Stability is depth- and length-dependent. Small per-layer phase misalignments compound multiplicatively, producing DC-component decay and attention collapse if the base is too low.
    \item Precision imposes a hard ceiling. FP32 limits create a “Precision Wall” for extremely long sequences, beyond which positional distinctions are erased regardless of model depth or training.
    \item Empirical heuristics are insufficient. Historical rules of thumb cannot predict instability in deeper or longer-context models.
    \item Case studies validate theory. LLaMA-2, LLaMA-3, and DeepSeek variants demonstrate that observed successes and failures align closely with the theoretical bounds.
\end{itemize}

Future Directions

\begin{itemize}
    \item Mixed-Precision or Novel Encodings: Extending the usable context range beyond FP32 limits.
    \item Adaptive RoPE or Layer-Wise Scaling: Dynamically adjusting the base across layers for improved stability.
    \item Integration with Retrieval and Memory-Augmented Models: Investigating how RoPE bounds interact with retrieval mechanisms and long-term memory modules.
    \item Benchmarking and Automated Verification: Tools to automatically compute stability and precision bounds for arbitrary architectures.
\end{itemize}

Overall, our framework provides practical, theoretically grounded guidance for long-context transformer design, retrofitting, and evaluation. By explicitly considering base, depth, and precision, researchers and practitioners can ensure stable attention, prevent aliasing, and unlock new capabilities in extended-context modeling.

\bibliographystyle{tmlr}
\bibliography{references}

\appendix
\section{Notation, Assumptions, and Model}
We work with the standard RoPE-style construction applied to the query/key vectors as rotations in paired dimensions.
\begin{itemize}
    \item Let the model dimension be $d$. Define $m = d/2$ complex-valued oscillator channels, indexed by $i \in \{0, \ldots, m-1\}$.
    \item Let the position index be $p \in \{0, \ldots, L-1\}$, where $L$ is the maximum context length.
    \item For each channel $k$, define an angular increment (per-token phase) $\theta_k > 0$. The RoPE-modulated complex phasor at position $p$ and channel $k$ is $e^{ip\theta_k}$.
    \item We assume the standard geometric-frequency parameterization: there exists a base $\text{base} > 1$ and a reference angular increment $\theta_0 = 1$ such that \[ \theta_k = \theta_0\, \text{base}^{-2(k-1)/d}, \qquad k = 1, 2, \ldots, d/2. \] This makes the per-token phase increments decrease geometrically across channels. The smallest angular increment is \[ \theta_{d/2} = \theta_0\, \text{base}^{-2(d/2 - 1)/d} = \theta_0\, \text{base}^{-(d-2)/d}. \]
    \item A period $T > 0$ for channel $k$ is any positive integer satisfying \[ \theta_k T \bmod 2\pi = 0, \] equivalently, \[ e^{iT\theta_k} = 1. \] The smallest such positive integer $T$ is the fundamental period of that channel. In continuous terms, \[ T_k = \frac{2\pi}{\gcd(\theta_k,\, 2\pi)}, \] and in practice the condition is simply \[ \theta_k T \in 2\pi \mathbb{Z}. \]
    \item The global (vector) positional encoding at position $p$ is the 
        $d/2$–dimensional vector of complex phasors

        \[
        \mathbf{v}(p)
        = \bigl[\, e^{ip\theta_1},\; e^{ip\theta_2},\; \ldots,\; e^{ip\theta_{d/2}} \,\bigr].
        \]

        Two positions $p$ and $q$ are \emph{indistinguishable} under RoPE if

        \[
        \mathbf{v}(p) = \mathbf{v}(q)
        \quad\text{(componentwise equality)}.
        \]

        More generally, positions $p$ and $q$ are \emph{ambiguous} if 
        $\mathbf{v}(p)$ and $\mathbf{v}(q)$ differ only by a common global phase.  
        For our purposes, we use the strict componentwise equality condition above, 
        as it provides a clear necessary condition for ambiguity.
\end{itemize}

\paragraph{Fundamental oscillator.}
The \emph{fundamental oscillator} is the slowest oscillator (i.e., the one
with the smallest angular increment $\theta_k$). Because the slowest
oscillator is least prone to aliasing, it serves as the global reference.

We define
\[\theta_{\min} = \min_k \theta_k = \theta_{d/2} \approx \frac{1}{\text{base}},\]

and its corresponding period
\[T_{\min} = \frac{2\pi}{\theta_{\min}} \approx 2\pi \cdot \text{base}.\]

\section{Proof for Theorem 1 (Fundamental Aliasing Limit)}

\begin{lemma}[Periodicity leads to indistinguishability]
If there exists a positive integer $T \le L$ such that for some channel $k$,
\[\theta_k T \bmod 2\pi = 0,\]

then for every position $p$ we have
\[\mathbf{v}(p) = \mathbf{v}(p+T)\]

for that channel. Consequently, the positional encoding is not injective on
$[0, L]$; it repeats with period $T$.
\end{lemma}

\begin{proof}
If $\theta_k T \in 2\pi \mathbb{Z}$, then
\[e^{i(p+T)\theta_k}
  = e^{ip\theta_k} \cdot e^{iT\theta_k}
  = e^{ip\theta_k}.\]

Thus,
\[\mathbf{v}(p) = \mathbf{v}(p+T)\]
for every $p$. \qedhere
\end{proof}

This lemma gives us a mechanistic way to show non-injectivity (aliasing):
find an integer period $T \le L$ for the fundamental oscillator channels.

As mentioned in Section~2, for RoPE the fundamental oscillator is the
component with the lowest frequency,
\[\theta_{\min} = \min_k(\theta_k) = \theta_{d/2} \approx \frac{1}{\text{base}}.\]

We define the primary period of the fundamental oscillator as
\[T_{\min} = \frac{2\pi}{\theta_{\min}}
          \approx 2\pi \cdot \text{base}.\]

According to Lemma~B.1, a necessary condition (under the componentwise
equality notion) for avoiding guaranteed periodic collisions due solely to
global common-period repetition is

\[L \le T_{\min}.\]

Conversely, if $L > T_{\min}$, then there exist distinct positions within
$[0, L]$ that match exactly across the fundamental oscillator channels
(i.e., positional encodings repeat), so the positional encoding is not
injective. \qedhere

\section{Proof for Theorem 2- Single-layer DC-component stability limit}
Consider two positions $p$ and $q$ with separation 
$\delta = |p - q|$.  Let queries and keys be complex-valued:

\[q_k(p) = q_k\, e^{ip\theta_k}, \qquad k_k(q) = k_k\, e^{iq\theta_k}.\]

The dot-product contribution at frequency $k$ is
\[q_k(p)\, \overline{k_k(q)} = q_k\, \overline{k_k}\, e^{i(p - q)\theta_k}.\]

The DC component corresponds to $(p - q)\theta_k = 0$, which occurs when
\begin{enumerate}
\item $p = q$ (self-attention), or
\item $\theta_k \to 0$ (very low-frequency dimensions, with the lowest
      frequency $\theta_{\min} = 1/\text{base}$).
\end{enumerate}

For RoPE, the DC component corresponds to the lowest-frequency oscillator,
i.e., the fundamental oscillator with
\[\theta_{\min} = \frac{1}{\text{base}}.\]

The real part of the RoPE attention kernel becomes
\[\Re\!\bigl(q_k \overline{k_k}\bigr)\,
    \cos\!\bigl((p - q)\theta_k\bigr)
  = \Re\!\bigl(q_k \overline{k_k}\bigr)\,
    \cos(\delta\,\theta_k),
\]
where $\delta = |p - q|$.

\emph{Note.} The imaginary term

\[
\Im\!\bigl(q_k \overline{k_k}\bigr)\,
    \sin\!\bigl((p - q)\theta_k\bigr)
\]

vanishes under the symmetry condition
\[
q_k(p)\,\overline{k_k(q)}
  = q_k(q)\,\overline{k_k(p)}.
\]

DC-component stability for the fundamental oscillator
$\theta_{\min} = 1/\text{base}$ means that its contribution remains
stable over a distance $\delta$ if, for some tolerance
$\epsilon \in (0,1)$,

\[\cos(\delta\,\theta_{\min}) \ge \epsilon.\]

For uniform stability (no sign flips) over all
$\delta \in [0, L]$, we require

\[\delta\,\theta_{\min} \le \arccos(\epsilon).\]

Substituting $\theta_{\min} = 1/\text{base}$ gives
\[\text{base} \ge \frac{\delta}{\arccos(\epsilon)}.\]

The worst case occurs at $\delta = L$, yielding the condition
\[\text{base} \ge \frac{L}{\arccos(\epsilon)}.\]

\section{Proof for Theorem 3 (Compounding multi-layer DC-component stability Limit)}
Because transformer layers are composed sequentially and each layer applies
RoPE independently, the contribution of the DC‑aligned component can be
modeled as a product of per‑layer attenuation factors.  If these factors are
independent and identically distributed and characterized by a cosine
alignment term
\[\rho = \cos(\delta\,\theta_k),\]

then the DC‑aligned signal decays multiplicatively with depth.  Thus, small
angular misalignments at each layer compound exponentially across layers,
yielding a stricter upper bound on the allowable phase rotation for
DC‑component stability.

Accordingly, the DC‑component stability condition becomes
\[\rho^N = \bigl(\cos(\delta\,\theta_{\min})\bigr)^N \ge \epsilon,\]

which is equivalent to
\[\cos(\delta\,\theta_{\min}) \ge \epsilon^{1/N}.\]

Repeating the same steps as in Appendix~C, we obtain
\[\text{base} \ge \frac{\delta}{\arccos(\epsilon^{1/N})}.\]

The worst case occurs at $\delta = L$, giving the depth‑aware stability
condition
\[\text{base} \ge \frac{L}{\arccos(\epsilon^{1/N})}. \qquad\qedhere\]

\textbf{Intuition}. In deep transformers, RoPE-induced phase errors do not average out across layers. Instead, each layer attenuates the DC-aligned component multiplicatively. As a result, small per-layer angular decoherence compounds exponentially with depth, tightening the RoPE base requirement as depth increases.

\section{Proof for Theorem 4 (Combined theoretical lower bound on the RoPE base)}
Combining the two sufficient conditions we derived:

\begin{itemize}
\item From Theorem~1: to avoid global position–signal aliasing, we require
\[\text{base} > \frac{L}{2\pi}.\]
\item From Theorem~3: to maintain stable positional coherence $\epsilon$
over a context length $L$, we require
\[\text{base} \ge \frac{L}{\arccos(\epsilon^{1/N})}.\]
\end{itemize}

Thus,
\[
\text{base}_{\min}
  = \max\!\left(
      \frac{L}{2\pi},
      \frac{L}{\arccos(\epsilon^{1/N})}
    \right)
\]

satisfies both conditions simultaneously and yields the stated combined
lower bound. \qedhere

\section{Proof of Theorem 5: The Numerical Erasure Bound}
In any digital Phase Modulation (PM) system, for two adjacent positions $p$ and $p+1$ to be distinguishable, the difference in their mapped values must exceed the minimum resolution of the hardware. In RoPE, the positional information is encoded as a rotation on the unit circle. The smallest phase increment happens between adjacent steps for the fundamental (slowest) oscillator, which is:
\[
\Delta\theta
  = (p+1)\theta_{\min} - p\theta_{\min}
  = \theta_{\min}
  = \frac{1}{\text{base}}.
\]

For the hardware to register a change in position at index $p$, we require
\[
f(p + \Delta\theta) \ne f(p).
\]

Floating‑point numbers (IEEE~754) represent values using a fixed‑length
mantissa.  The machine epsilon $\epsilon_{\mathrm{mach}}$ denotes the
smallest relative spacing between two representable floating‑point numbers.
To maintain positional coherence, the phase increment must exceed the
hardware resolution:
\[
\Delta\theta > \epsilon_{\mathrm{mach}}.
\]

Substituting $\Delta\theta = 1/\text{base}$ into the hardware‑resolution
condition $\Delta\theta > \epsilon_{\mathrm{mach}}$ gives
\[
\frac{1}{\text{base}} > \epsilon_{\mathrm{mach}}.
\]

Rearranging yields
\[
\text{base} < \frac{1}{\epsilon_{\mathrm{mach}}}.
\]

This inequality defines the \emph{Numerical Erasure Wall}: if the base
exceeds this limit, the hardware can no longer distinguish between
positions $p$ and $p+1$, causing a collapse of relative positional
awareness.

\end{document}